\documentclass[journal]{IEEEtran}
\usepackage{cite}

\usepackage{bm}
\usepackage{color}
\usepackage[pdftex]{graphicx}
\usepackage{siunitx}
\ifCLASSINFOpdf
\else
\fi
\usepackage{amsmath}
\usepackage{amsthm}
\usepackage{setspace}
\usepackage{multirow}
\usepackage{amssymb}
\usepackage{mathrsfs}
\usepackage{algorithmic}
\usepackage{algorithm}
\makeatletter
\newcommand{\removelatexerror}{\let\@latex@error\@gobble}
\makeatother

\begin{document}
%
\title{Conformal Loss-Controlling Prediction}
%
%
%

\author{Di Wang, Ping Wang, Zhong Ji, \textit{Senior Member, IEEE}, Xiaojun Yang, Hongyue Li 
  \thanks{This work has been submitted to the IEEE for possible publication. Copyright may be transferred without notice, after which this version may no longer be accessible.}
  \thanks{This work was supported by the National Natural Science
Foundation of China under Grant 62106169. (Corresponding author: Hongyue Li)}
  \thanks{Di Wang and Zhong Ji are with the School of Electrical and Information Engineering, Tianjin University, Tianjin 300072, China, 
and also with the Tianjin Key Laboratory of Brain-inspired Intelligence Technology, School of Electrical and Information Engineering, Tianjin University, Tianjin 300072, China.
(email: wangdi2015@tju.edu.cn; jizhong@tju.edu.cn;).
 }
  \thanks{Ping Wang and Hongyue Li are with the School of Electrical and Information Engineering, Tianjin University, Tianjin 300072, China.
(email: wangps@tju.edu.cn; lihongyue@tju.edu.cn).
 }
  \thanks{Xiaojun Yang is with the Tianjin Meteorological Observatory, Tianjin 300074, China. (email: boluo0127@yeah.net)}
}
%
%

\markboth{This work has been submitted to the IEEE for possible publication}%
{Shell \MakeLowercase{\textit{et al.}}: Bare Demo of IEEEtran.cls for IEEE Journals}
%




\maketitle

\begin{abstract}
Conformal prediction is a learning framework controlling prediction coverage of prediction sets, which can be built on any learning algorithm for point prediction. This work proposes a learning framework named conformal loss-controlling prediction, which extends conformal prediction to the situation where the value of a loss function needs to be controlled. Different from existing works about risk-controlling prediction sets and conformal risk control with the purpose of controlling the expected values of loss functions, the proposed approach in this paper focuses on the loss for any test object, which is an extension of conformal prediction from miscoverage loss to some general loss. The controlling guarantee is proved under the assumption of exchangeability of data in finite-sample cases and the framework is tested empirically for classification with a class-varying loss and statistical postprocessing of numerical weather forecasting  applications, which are introduced as point-wise classification and point-wise regression problems. All theoretical analysis and experimental results confirm the effectiveness of our loss-controlling approach.
\end{abstract}


%

\section{Introduction}

Prediction sets convey uncertainty or confidence information for users, which is more preferred than prediction points, especially for sensitive applications such as medicine, finance and weather forecasting \cite{balasubramanian2014conformal} \cite{li2019short} \cite{wang2021conformal}. One example is constructing prediction intervals with confidence $1 - \delta$ for regression problems, where the statistical guarantee is expected such that the true labels are covered in probability $1 - \delta$ \cite{morales2023dual}. Nowadays, many researches have been proposed to build set predictors. Bayesian methods \cite{lu2021hierarchical} and Gaussian process \cite{gomez2023adaptive} are straightforward ways of producing prediction sets based on posterior distributions. However, their prediction sets can be misleading if the prior assumptions are not correct, which is often the case since the prior is usually unknown in applications \cite{melluish2001comparing} \cite{papadopoulos2023guaranteed}. Other statistical methods such as bootstrap-based methods \cite{zhang2023long} and quantile regression \cite{koenker1978regression} are also able to output prediction sets for test labels, but their coverage guarantees can only be obtained in the asymptotic setting, and the prediction sets may fail to cover the labels frequently in finite-sample cases. Different from these works, conformal prediction (CP), a promising non-parametric learning framework aiming to provide reliable prediction sets, can provide the finite-sample coverage guarantee only under the assumption of exchangeability of data samples \cite{vovk2005algorithmic}. This property of validity has been proved both theoretically and empirically in many works and applied to many areas \cite{angelopoulos2021gentle} \cite{fontana2023conformal}. Besides, many researches extend CP to more general cases, such as conformal prediction for multi-label learning \cite{wang2015comparison} \cite{messoudi2021copula}, functional data \cite{lei2015conformal} \cite{diquigiovanni2022conformal}, few-shot learning \cite{fisch2021few}, distribution shift \cite{tibshirani2019conformal} \cite{barber2022conformal} and time series \cite{jensen2022ensemble} \cite{zaffran2022adaptive}.

However, the researches about set predictors mentioned above mainly make promise about the coverage of prediction sets, i.e., they only control the miscoverage loss of set predictors, which can not be applied to other broad applications concerning controlling general losses. For example, consider classifying MRI images into several diagnostic categories \cite{bates2021distribution}, where different categories cause different consequence. In this setting, the loss of the true label $y$ being not included in the prediction set should be dependent on $y$, which is the problem of classification with a class-varying loss. Another example is tumor segmentation \cite{angelopoulos2022conformal}. Instead of making prediction sets to overly cover the pixels of tumor, one may care more about controlling other losses such as false negative rate. Other practical settings include controlling the $l_1$ projective distance for protein structure prediction, controlling a hierarchical distance for hierarchical classification and controlling F1-score for open-domain question answering \cite{bates2021distribution} \cite{angelopoulos2022conformal}. In these applications, the prediction sets with the coverage guarantee are not useful, as they are not constructed with controlling these general losses in mind.

To tackle this issue, two works for extending the finite-sample coverage guarantee of CP have been proposed recently. One is the work of conformal prediction sets with limited false positives (CPS-LFP) \cite{fisch2022conformal}. It employs DeepSets \cite{zaheer2017deep} to estimate the expected value or the cumulative distribution function of the number of false positives, and then uses calibration data to control the number of false positives of prediction sets. Conformal risk control (CRC) \cite{angelopoulos2022conformal} extends CP to prediction tasks of controlling the expected value of a general loss based on finding the optimal parameter for nested prediction sets. The spirit is to employ calibration data to obtain the information of the upper bound of the expected value of the loss function at hand and control the expected value for the test object, whose main idea was originally proposed from their pioneer work named risk-controlling prediction sets (RCPS) \cite{bates2021distribution}. CRC and RCPS aim to control the expected value instead of the value of a general loss for set predictors. By contrast, CPS-LFP can control the value of the loss related to false positives, but it is not general enough.

In some applications, controlling the value of a general loss can be more preferred than controlling the expected value, since one may only care about the loss value for a specific test object, just like the coverage guarantee made by CP and the $(k,\delta)$-FP validity acheived by CPS-LFP. Therefore, this paper extends CP to the situation where the value of a general loss needs to be controlled, which has been not considered in the literature to our best knowledge. Our approach is similar to CRC with the main difference being that we focus on finding the optimal parameter for nested prediction sets to control the loss. Therefore, we also concentrate on inductive conformal prediction \cite{papadopoulos2008inductive} or split conformal prediction \cite{lei2018distribution} process like CRC.

Recall that inductive conformal prediction makes the coverage guarantee as follows,
\begin{equation}\nonumber
P \Bigg (Y_{n+1} \in C^{(n)}_{1-\delta}(X_{n+1}) \Bigg ) \geq 1 - \delta,
\end{equation}
where $\delta$ is the significance level preset by users, $C^{(n)}_{1-\delta}$ is the set predictor made by CP based on $n$ calibration data $\{(X_i, Y_i)\}_{i=1}^n$,  $(X_{n+1}, Y_{n+1})$ is the test feature-response pair, and the randomness is from both $\{(X_i, Y_i)\}_{i=1}^n$ and $(X_{n+1}, Y_{n+1})$.
By comparison, conformal loss-controlling prediction (CLCP), the learning framework proposed in this paper, provides the controlling guarantee as follows,
\begin{equation}\nonumber
P \Bigg (L \Big (Y_{n+1}, C_{\lambda^*}(X_{n+1}) \Big ) \leq \alpha \Bigg ) \geq 1 - \delta,
\end{equation}
where $L$ is a loss function satisfying some monotonic conditions as in \cite{angelopoulos2022conformal}, $\alpha$ is the preset level of loss, $C_{\lambda}$ is a set predictor usually constructed by an underlying algorithm and a parameter $\lambda$. The optimal $\lambda^*$ is obtained based on $\alpha$, $\delta$ and calibration data. The controlling guarantee needs two levels $\alpha$ and $\delta$ to be chosen by users, which is similar with that in \cite{bates2021distribution}, i.e., CLCP guarantees that the prediction loss is not greater than $\alpha$ with high probability $1 - \delta$ when $\delta$ is small such as $0.1$. If $L$ is defined based on false positives for multi-label classification, the controlling guarantee above can be seen as the ($\alpha, \delta$)-FP validity defined in Definition 4.2 in \cite{fisch2022conformal}.

We prove the controlling guarantee for distribution-free and finite-sample settings with the assumption of exchangeability of data samples. The main idea is that we find the $\lambda^*$ to make the $1 - \delta$ quantile of the loss values on calibration data not greater than $\alpha$, which is inspired by CRC focusing on making the mean of the loss values not greater than $\alpha$. Since the property of the set predictors and loss functions used in CLCP is the same as that used in CRC, CLCP can also be applied to many applications concerning controlling general losses. These applications include not only the areas about classification and image segmentation, but also the field of graph signal processing \cite{ortega2018graph} \cite{wu2020comprehensive}, for example, protein structure prediction.

The proposed CLCP is a novel learning framework compared to existing researches. Different from those aiming to control the value of the miscoverage loss, CLCP is a more general approach for the purpose of controlling the value of a general loss. Besides, CLCP can be widely used for many situations whereas CPS-LFP is specifically designed for controlling the loss related to false negatives. Also, CLCP differs from CRC and RCPS as their purpose is to control the expected value instead. Therefore, in the experimental section, we concentrate on designing the experiments to verify the theoretical conclusion for different applications, as the idea of controlling general losses for set predictors is original. To be specific, we test our proposed CLCP in classification with a class-varying loss introduced in \cite{bates2021distribution}, and postprocessing of numerical weather forecasts, which we consider as point-wise classification and point-wise regression problems. The experimental results empirically confirm the theoretical guarantee we prove in this paper.

In summary, the main contributions of this paper are:
\begin{itemize}
\item A learning framework named conformal loss-controlling prediction (CLCP) is proposed for controlling the prediction loss for the test object. The approach is simple to implement and can be built on any machine learning algorithm for point prediction.
\item The controlling guarantee is proved mathematically for finite-sample cases with the exchangeability assumption, without any further assumption for data distribution.
\item The controlling guarantee is empirically verified by classification with a class-varying loss and weather forecasting problems, which confirms the effectiveness of CLCP.
\end{itemize}

The rest of this paper is organized as follows. Section II reviews inductive conformal prediction and conformal risk control.  Section III introduces conformal loss-controlling prediction and its theoretical guarantee. Section IV conducts experiments to test the proposed method and the conclusions are drawn in Section V.

\section{Inductive Conformal Prediction and Conformal Risk Control}

This section reviews inductive conformal prediction and conformal risk control. Throughout this paper, $\{(X_i, Y_i)\}_{i = 1}^{n+1}$ denotes $n+1$ data drawn exchangeably from $P_{XY}$ on $\mathcal{X} \times \mathcal{Y}$, where $\{(X_i, Y_i)\}_{i = 1}^{n}$ is the calibration dataset and $(X_{n+1}, Y_{n+1})$ is the test object-response pair. We use lower-case letter $(x_i, y_i)$ to represent the realization of $(X_i, Y_i)$.

The set-valued function and loss function considered in this paper are the same as those in \cite{angelopoulos2022conformal} and \cite{bates2021distribution}, which we formally introduce as follows.
Let $C_{\lambda}: \mathcal{X} \rightarrow \mathcal{Y}'$ be a set-valued function with a parameter $\lambda \in \mathcal{R}$, where $\mathcal{Y}'$ represents some space of sets and $\mathcal{R}$ is the set of real numbers. Taking single-label classification for example, $\mathcal{Y}'$ can be the power set of $\mathcal{Y}$. For binary image segmentation, $\mathcal{Y}'$ can be equal to $\mathcal{Y}$ as the space of all possible results of image segmentation, where the sets here stand for all of the pixels of positive class for the image.

We also introduce the nesting property for prediction sets and losses as in \cite{bates2021distribution} as follows. For each realization of input object $x$, we assume that $C_{\lambda}(x)$ satisfies the following nesting property:
\begin{equation}
\lambda_1 < \lambda_2 \ \Longrightarrow\  C_{\lambda_1}(x) \subseteq C_{\lambda_2}(x).
\end{equation}
Furthermore, with $S_1$ and $S_2$ being two subsets of $\mathcal{Y}$, we assume that $L: \mathcal{Y} \times \mathcal{Y}' \rightarrow \mathcal{R}$ is a loss function respecting the following nesting property for each realization of response $y$:
\begin{equation}
S_1 \subseteq S_2 \subseteq \mathcal{Y}' \ \Longrightarrow \ L(y, S_2) \leq L(y, S_1) \leq B,
\end{equation}
where $B$ is the upper bound of  the loss function.

\subsection{Inductive Conformal Prediction}

Inductive conformal prediction (ICP) is a computationally efficient version of the original conformal prediction approach. It starts with any measurable function named nonconformity measure $A: \mathcal{X} \times \mathcal{Y} \rightarrow \mathcal{R}$ and obtains $n$ nonconformity scores as
\begin{equation}\nonumber
A_i = A(X_i, Y_i),
\end{equation}
for $i = 1, \cdots, n$.
Then, with the exchangeable assumption and a preset $\delta \in (0,1)$, one can conclude that
\begin{equation}\nonumber
P \Bigg ( A(X_{n+1}, Y_{n+1}) \leq Q^{(n)}_{1-\delta} \Bigg ) \geq 1 - \delta,
\end{equation}
where $Q^{(n)}_{1-\delta}$ is the $1 - \delta$ quantile of $\{A_i\}_{i=1}^n \cup \{\infty\}$ \cite{tibshirani2019conformal}. Therefore, the prediction set made by ICP is
\begin{equation}\nonumber
C^{(n)}_{1-\delta}(X_{n+1}) = \{ y : A(X_{n+1}, y) \leq Q^{(n)}_{1-\delta}\},
\end{equation}
which satisfies
\begin{equation}\nonumber
P \Bigg (Y_{n+1} \in C^{(n)}_{1-\delta}(X_{n+1}) \Bigg ) \geq 1 - \delta.
\end{equation}

The nonconformity measure $A$ is often defined based on a point prediction model $\hat{f}$ learned from some other training samples, each of which is also drawn from $P_{XY}$.

Here is an example of constructing prediction sets with CP. For a classification problem with $K$ classes, one can first train a classifier $\hat{f}: \mathcal{X} \rightarrow [0,1]^K$ with the $i$th output being the estimation of the probability of the $i$th class, and calculate the nonconformity scores as
\begin{equation}\nonumber
A(x, y) = 1 -\hat{f}_k(x),
\end{equation}
where $\hat{f}_k$ is the $k$th output of $\hat{f}(x)$, if $y$ stands for the $k$th class. Therefore, the corresponding prediction set for an input object $x$ is
\begin{equation}\nonumber
C^{(n)}_{1-\delta}(x) = \{ k : \hat{f}_k(x) \geq 1 - Q^{(n)}_{1-\delta}\},
\end{equation}
which indicates that $k \in C^{(n)}_{1-\delta}(x)$ if the estimated probability of $k$th class is not less than $1 - Q^{(n)}_{1-\delta}$.

\subsection{Conformal Risk Control}

Different from conformal prediction, CRC starts with a set-valued function with the nesting property, whose approach is inspired by nested conformal prediction \cite{gupta2022nested} and was first proposed in the researches about risk-controlling prediction sets.

Assume one has a way of constructing a set-valued function $C_{\lambda}$ with the nesting property of formula (1). Given a loss function $L$ with the nesting property of formula (2), the purpose of CRC is to find $\lambda^*$ such that
\begin{equation}
E \Big [L(Y_{n+1}, C_{\lambda^*}(X_{n+1})) \Big ] \leq \alpha,
\end{equation}
i.e., the expected loss or the risk is not greater than $\alpha$.

To do so, CRC \textbf{first} calculates $L_i(\lambda)$ as
\begin{equation}
L_i(\lambda) = L(Y_i, C_{\lambda}(X_i)),
\end{equation}
with the fact that $L_i(\lambda)$ is a monotone decreasing function of $\lambda$ based on the nesting properties. \textbf{Then}, CRC searches for $\lambda^*$ using the following equation,
\begin{equation}\nonumber
\lambda^* = \inf \Bigg \{ \lambda: \frac{n}{n+1}\hat{R}_n(\lambda) + \frac{B}{n+1} \leq \alpha    \Bigg \},
\end{equation}
where $\hat{R}_n(\lambda) = (L_1(\lambda)+\cdots+L_n(\lambda))/n $ is an estimation of the risk on calibration data and $B$ is introduced to make the estimation not overconfident. 

These two steps of CRC are too simple that one may surprise about its theoretical conclusion that with the assumption of exchangeability of data samples, the prediction set
$ C_{\lambda^*}(X_{n+1})$ obtained by CRC satisfies formula (3), which has been also proved empirically in \cite{angelopoulos2022conformal}. CRC extends CP from controlling the expected value of miscoverage loss to some general loss, which can be applied to the cases where $\mathcal{Y}$ is beyond real numbers or vectors, such as images, fields and even graphs.

After tackling the theoretical issue, the problem for CRC is how to construct  $C_{\lambda}$. Here, we also give an example of a classification problem with $K$ classes. In fact, with the same notations of the example in Section II-A, CRC can construct the prediction set as
\begin{equation}\nonumber
C_{\lambda}(x) = \{ k : \hat{f}_k(x) \geq 1 - \lambda\}.
\end{equation}
Therefore, as long as $L$ satisfies formula (2), such as $L$ is the indicator of miscoverage, CRC guarantees to control the risk as formula (3).

\section{Conformal Loss-Controlling Prediction And Its Theoretical Analysis}

This section introduces the approach of CLCP and its theoretical analysis. CLCP also has two steps like CRC, and the main difference between them is that CLCP focuses on whether the estimation of the $1 - \delta$ quantile of the losses is not greater than $\alpha$ while CRC concentrates on whether the mean of the losses not greater than $\alpha$. The controlling of the $1 - \delta$ quantile of the losses makes CLCP able to control the value of a general loss by employing the probability inequation derived from the exchangeability assumption, which is also employed by ICP if the loss is seen as the nonconformity score.

Suppose one has a way of constructing a set-valued function $C_{\lambda}$ with the nesting property of formula (1), which can be the same as that used in CRC. Here, we assume that the parameter $\lambda$ is selected from a discrete set $\Lambda$, such as from $0$ to $1$ with a step size $0.01$, which avoids us from the assumption of right continuous for the loss function in theoretical analysis, and is also reasonable since we actually search for $\lambda^*$ with some step size in practice \cite{angelopoulos2022conformal}  \cite{bates2021distribution}. Besides, the latest paper about risk-controlling prediction also makes this discrete assumption for general cases \cite{angelopoulos2021learn}. After determining $C_{\lambda}$  and $\Lambda$, CLCP \textbf{first} calculates $L_i(\lambda)$ on calibration data as formula (4). \textbf{Then}, for any preset $\alpha \in \mathcal{R}$ and $\delta \in (0, 1)$, CLCP searches for $\lambda^*$ such that
\begin{equation}
\lambda^* = \min \Bigg \{ \lambda \in \Lambda: Q^{(n)}_{1-\delta}(\lambda) \leq \alpha    \Bigg \},
\end{equation}
with $Q^{(n)}_{1-\delta}(\lambda)$ being the $1 - \delta$ quantile of $\{L_i(\lambda)\}_{i=1}^n \cup \{ B \}$. The approach of CLCP is summarised in Algorithm 1, which is easy to implement.
\begin{algorithm}
\caption{Conformal Loss-Controlling Prediction}
\label{alg:Framwork}
\begin{algorithmic}[1]
\REQUIRE ~~\\
Calibration dataset $\{(x_i, y_i)\}_{i=1}^n$, test input object $x_{n+1}$, the set predictor $C_{\lambda}$ satisfying formula (1), the loss function $L$ satisfying formula (2), preset $\alpha \in \mathcal{R}$ and $\delta \in (0, 1)$.\\
\ENSURE ~~\\
Predictive set for $y_{n+1}$.
\STATE
Based on formula (4), calculate $\{L_i(\lambda)\}_{i=1}^n$.
\STATE
Search for $\lambda^*$ satisfying formula (5).
\RETURN
$C_{\lambda^*}(x_{n+1})$
\end{algorithmic}
\end{algorithm}

Next, we introduce the definition of $(\alpha, \delta)$-loss-controlling set predictors and then prove our theoretical conclusion about CLCP.

\newtheorem{definition}{Definition}
\begin{definition}
Given a loss function $L: \mathcal{Y} \times \mathcal{Y}' \rightarrow \mathcal{R}$ and a random sample $(X, Y) \in \mathcal{X} \times \mathcal{Y}$, a random set-valued function $C$ whose realization is in the space of  functions $\mathcal{X} \rightarrow \mathcal{Y}'$ is a $(\alpha, \delta)$-loss-controlling set predictor if it satisfies that
\[
P \Bigg (L \Big (Y, C(X) \Big ) \leq \alpha \Bigg ) \geq 1 - \delta,
\]
where the randomness is both from $C$ and $(X, Y)$.
\end{definition}

After all these preparations, we can prove in Theorem 1 that $C_{\lambda^*}$ constructed by CLCP is a $(\alpha, \delta)$-loss-controlling set predictor.

\newtheorem{theorem}{Theorem}
\begin{theorem}
Suppose $\{(X_i, Y_i)\}_{i = 1}^{n+1}$ are $n+1$ data drawn exchangeably from $P_{XY}$ on $\mathcal{X} \times \mathcal{Y}$, $C_{\lambda}: \mathcal{X} \rightarrow \mathcal{Y}'$ is a set-valued function satisfying formula (1) with the parameter $\lambda$ taking values from a discrete set $ \Lambda \subset \mathcal{R}$ , $L: \mathcal{Y} \times \mathcal{Y}' \rightarrow \mathcal{R}$ is a loss function satisfying formula (2) and $L_i(\lambda)$ is defined as formula (4). For any preset $\alpha \in \mathcal{R}$, if $L$ also satisfies the following condition,
\begin{equation}
\min_{\lambda}\max_i L_i(\lambda) \leq \alpha,
\end{equation}
then for any $\delta \in (\frac{1}{n+1},1)$, we have
\begin{equation}
P \Bigg (L \Big (Y_{n+1}, C_{\lambda^*}(X_{n+1}) \Big ) \leq \alpha \Bigg ) \geq 1 - \delta,
\end{equation}
where $\lambda^*$ is defined as formula (5).
\end{theorem}

\begin{proof}
Let $Q_{1-\delta}^{(n+1)}(\lambda)$ be the  $1-\delta$ quantile of $\{L_i(\lambda)\}_{i=1}^{n+1}$, and
define $\tilde{\lambda}$ as
\begin{equation}\nonumber
\tilde{\lambda} = \min \Bigg  \{ \lambda \in \Lambda: Q_{1-\delta}^{(n+1)}(\lambda) \leq \alpha \Bigg  \}.
\end{equation}
Similarly, let $Q_{1-\delta}^{(n)}(\lambda)$ be the  $1-\delta$ quantile of $\{L_i(\lambda)\}_{i=1}^{n} \cup \{ B \}$, and
we have
\begin{equation}\nonumber
\lambda^* = \min \Bigg \{ \lambda \in \Lambda: Q^{(n)}_{1-\delta}(\lambda) \leq \alpha    \Bigg \}.
\end{equation}
As $\delta \in (\frac{1}{n+1}, 1)$ and formula (6) holds, $\tilde{\lambda}$ and  $\lambda^*$ are well defined.
Since $B$ is the upper bound of $L_{n+1}(\lambda)$, by definition, we have
\[
\tilde{\lambda} \leq \lambda^*,
\]
which leads to
\begin{equation}
L_{n+1}(\lambda^*) \leq L_{n+1}(\tilde{\lambda}),
\end{equation}
as $C_{\lambda}$ and $L$ satisfy the nesting properties of formula (1) and (2).

Since $\tilde{\lambda}$ is dependent on the whole dataset $\{(X_i, Y_i)\}_{i = 1}^{n+1}$, $\{L_i(\tilde{\lambda})\}_{i=1}^{n+1}$ are exchangeable variables, which leads to
\begin{equation}
P \Bigg (L_{n+1}(\tilde{\lambda}) \leq Q_{1-\delta}^{(n+1)}(\tilde{\lambda})  \Bigg ) \geq 1 - \delta,
\end{equation}
as $Q_{1-\delta}^{(n+1)}(\tilde{\lambda})$ is just the corresponding $1 - \delta$ quantile (See the proof of Lemma 1 in \cite{tibshirani2019conformal}).

Combining the definition of $\tilde{\lambda}$, formula (8) and (9), we have
\begin{equation}\nonumber
P \Bigg (L_{n+1}(\lambda^*) \leq \alpha \Bigg ) \geq 1 - \delta,
\end{equation}
which completes the proof.
\end{proof}

At the end of this section, we show that CP can be seen as a special case of CLCP from the following viewpoint. Suppose $C_{\lambda}$ is constructed by a nonconformity score $A$, which is defined as
\begin{equation}\nonumber
C_{\lambda}(x) = \{y : A(x, y) \leq \lambda\},
\end{equation}
and  $L$ is the miscoverage loss such that
\begin{equation}\nonumber
L_i(\lambda) = L(y_i, C_{\lambda}(x_i)) = \mathbb{I}\{y_i \notin C_{\lambda}(x_i)\},
\end{equation}
where $\mathbb{I}$ is the indicator function. In this case, $Q^{(n)}_{1-\delta}(\lambda)$ can only be $0$ or $1$ as the loss can only be these two numbers. Besides, only $\alpha \in [0, 1)$ is meaningful, which means that one wants to control the miscoverage. For CLCP, let $\Lambda$ be an arithmetic sequence whose common difference,  minimum and maximum are $\Delta$, $\lambda_{min}$ and $\lambda_{max}$ respectively and set $\alpha = 0$. By definition, $\lambda^*$ can be written as
\begin{align}
\lambda^* & = \min \Bigg \{ \lambda \in \Lambda: \frac{1}{n+1}\sum_{i=1}^{n} \mathbb{I}\{a_i \leq \lambda \} \geq 1 - \delta  \Bigg \} \nonumber \\
&  = \min \Bigg \{ \lambda \in \Lambda: \frac{1}{n}\sum_{i=1}^{n} \mathbb{I}\{a_i \leq \lambda \} \geq \frac{\lceil (1 - \delta)(n+1) \rceil}{n}  \Bigg \}, \nonumber
\end{align}
where $a_i = A(x_i, y_i)$ is the nonconformity score of the $i$th calibration data for CP.
In comparison, referring to \cite{angelopoulos2022conformal}, the optimal $\hat{\lambda}$ for CP is
\begin{align}
\hat{\lambda} = \inf \Bigg \{ \lambda \in \mathcal{R}: \frac{1}{n}\sum_{i=1}^{n} \mathbb{I}\{a_i \leq \lambda \} \geq \frac{\lceil (1 - \delta)(n+1) \rceil}{n}  \Bigg \}. \nonumber
\end{align}
Therefore, if $\lambda_{min} < a_i < \lambda_{max}$ for each $i$, we have
\begin{equation}\nonumber
|\lambda^* - \hat{\lambda}| \leq \Delta,
\end{equation}
which implies that the prediction sets of CP and CLCP are nearly the same if $\Delta$ is small enough. In summary, if $C_{\lambda}$ and $L$ have special forms and $\Lambda$ includes the upper and lower bounds of nonconformity scores with $\Delta$ being small enough to be ignored, CP can be seen as a special case of CLCP.

\section{Experiments}

This section conducts the experiments to empirically test the approach of CLCP. First, we build CLCP for the classification problem with a class-varying loss introduced in \cite{bates2021distribution}. Then, we focus on two types of weather forecasting applications, which can be seen as point-wise classification and point-wise regression problems respectively. All experiments were coded in Python \cite{citepython}. The statistical learning methods used in Section IV-A were implemented using Scikit-learn \cite{scikit-learn} and the deep learning methods used in Section IV-B and Section IV-C were implemented with Pytorch \cite{NEURIPS2019_9015}.

\subsection{CLCP for classification with a class-varying loss}

We collected $20$ binary or multiclass classification datasets from UCI repositories \cite{asuncion2007uci} whose information is summarized in Table I. The problem is to make the prediction sets of labels controlling the following loss
\begin{equation}\nonumber
L(y, C) = L_y \mathbb{I} \{y \notin C\},
\end{equation}
where $L_y$ is the loss for $y$ being not in the prediction set $C$. The loss for each label is generated uniformly on $(0,1)$ like \cite{bates2021distribution}. Support vector machine (SVM) \cite{platt1999probabilistic}, neural network (NN) \cite{schmidhuber2015deep} and random forests (RF) \cite{breiman2001random} were employed as the underlying algorithms separately to construct prediction sets based on CLCP. The prediction set $C_{\lambda}$ is constructed as
\begin{equation}\nonumber
C_{\lambda}(x) = \{ k : \hat{f}_k(x) \geq 1 - \lambda\},
\end{equation}
where $\hat{f}_k$ is the estimated probability of the observation being $k$th class by the corresponding underlying algorithm. For each dataset, we used $20\%$ of the data for testing and $80\%$ and $20\%$ of the remaining data for training and calibration respectively. Based on the training data, we selected the meta-parameters with three-fold cross-validation and used the optimal meta-parameters to train the classifiers. The regularization parameter of SVM was selected from $\{0.001, 0.01, 0.1, 1, 10, 100\}$, and the learning rate and the epochs of NN were selected from $\{0.001, 0.0001\}$ and $\{200, 500, 1000\}$. The number of trees of RF were selected from $\{100, 300, 500\}$ and the partition criterion was either gini or entropy. After training, we used the trained classifiers and the calibration data to search for $\lambda^*$ with Algorithm 1 and construct the final set predictors. All of the features were normalized to $[0, 1]$ by min–max normalization and for each dataset, the experiments were conducted $10$ times and the average results were recorded.

The bar plots in Fig. 1 and Fig. 2 show the experimental results for $20$ public datasets with $\delta \in \{0.05, 0.1, 0.15, 0.2\}$ and $\alpha \in \{0.1, 0.2\}$. The results in Fig. 1 concern about the frequency of the prediction losses being greater than $\alpha$ on test set, which is the estimated probability of
\begin{equation}\nonumber
P \Bigg (L \Big (Y_{n+1}, C_{\lambda^*}(X_{n+1}) \Big ) > \alpha \Bigg ),
\end{equation}
and should be near or lower than $\delta$ empirically due to formula (7).
The bar plots of Fig. 1 demonstrate that the frequency of the prediction losses being greater than $\alpha$ is near or below $\delta$, which verifies the conclusion of Theorem 1.

The bar plots of Fig. 2 show the average sizes of prediction sets for different $\delta$, describing the informational efficiency of the prediction sets. Changing $\delta$ can effectively change the average size of prediction sets and changing $\alpha$ may slightly change average size (such as the results for wine-quality-red). Although many prediction sets are meaningful with average sizes being near $1$, the prediction sets for the dataset contrac may be not useful, since no matter how to change $\delta$ and $\alpha$, the average sizes of the prediction sets are all near or above $2$, whereas the number of classes of contrac is $3$. Thus, how to construct efficient prediction sets in the learning framework of CLCP is worth exploring for further researches.

Combining Fig. 1 and Fig. 2, we observe that different classifiers can perform differently for different datasets, which indicates that the underlying algorithm affects the performance and the model selection approach is necessary for CLCP.

\begin{table}
\centering
\caption{Datasets from UCI Repositories}
\scalebox{0.7}{
\begin{tabular}{lcccc}
\hline
Dataset & Examples & Dimensionality &  Classes \\
\hline
bc-wisc-diag & 569 & 30 & 2 \\
car & 1728 & 6 & 4 \\
chess-kr-kp & 3196 & 36 & 2 \\
contrac & 1473 & 9 & 3 \\
credit-a & 690 & 15 & 2 \\
credit-g & 1000 & 20 & 2 \\
ctg-10classes & 2126 & 21 & 10 \\
ctg-3classes & 2126 & 21 & 3 \\
haberman & 306 & 3 & 2 \\
optical & 5620 & 62 & 10 \\
phishing-web & 11055 & 30 & 2 \\
st-image & 2310 & 18 & 7 \\
st-landsat & 6435 & 36 & 6 \\
tic-tac-toe & 958 & 9 &  2 \\
wall-following & 5456 & 24 &  4\\
waveform & 5000& 21&  3\\
waveform-noise & 5000 & 40 &  3\\
wilt & 4839 & 5 &  2\\
wine-quality-red & 1599 & 11 &  6\\
wine-quality-white & 4898 & 11 &  7 \\
\hline
\end{tabular}}
\end{table}

\begin{figure*}[h]
\centering
\includegraphics[width = 0.75 \hsize]{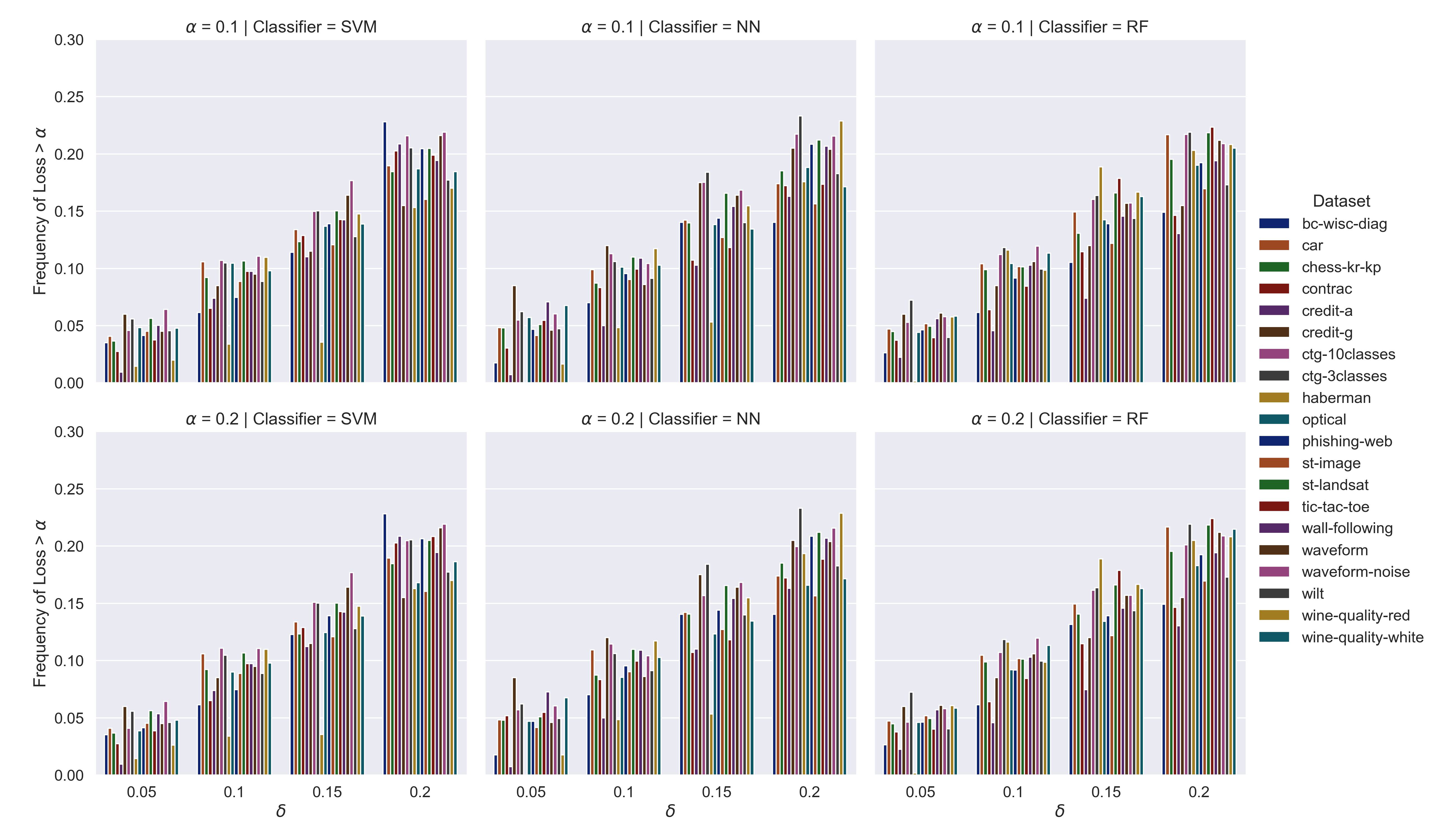}
\caption{Bar plots of the frequencies of the prediction losses being greater than $\alpha$ vs. $\delta = 0.05, 0.1, 0.15, 0.2$ on test data for classification with a class-varying loss. The first row corresponds to $\alpha = 0.1$ and the second row corresponds to $\alpha = 0.2$. Different columns represent different classifiers. All bars are near or below the preset $\delta$, which confirms the controlling guarantee of CLCP empirically.}
\end{figure*}

\begin{figure*}[h]
\centering
\includegraphics[width = 0.75 \hsize]{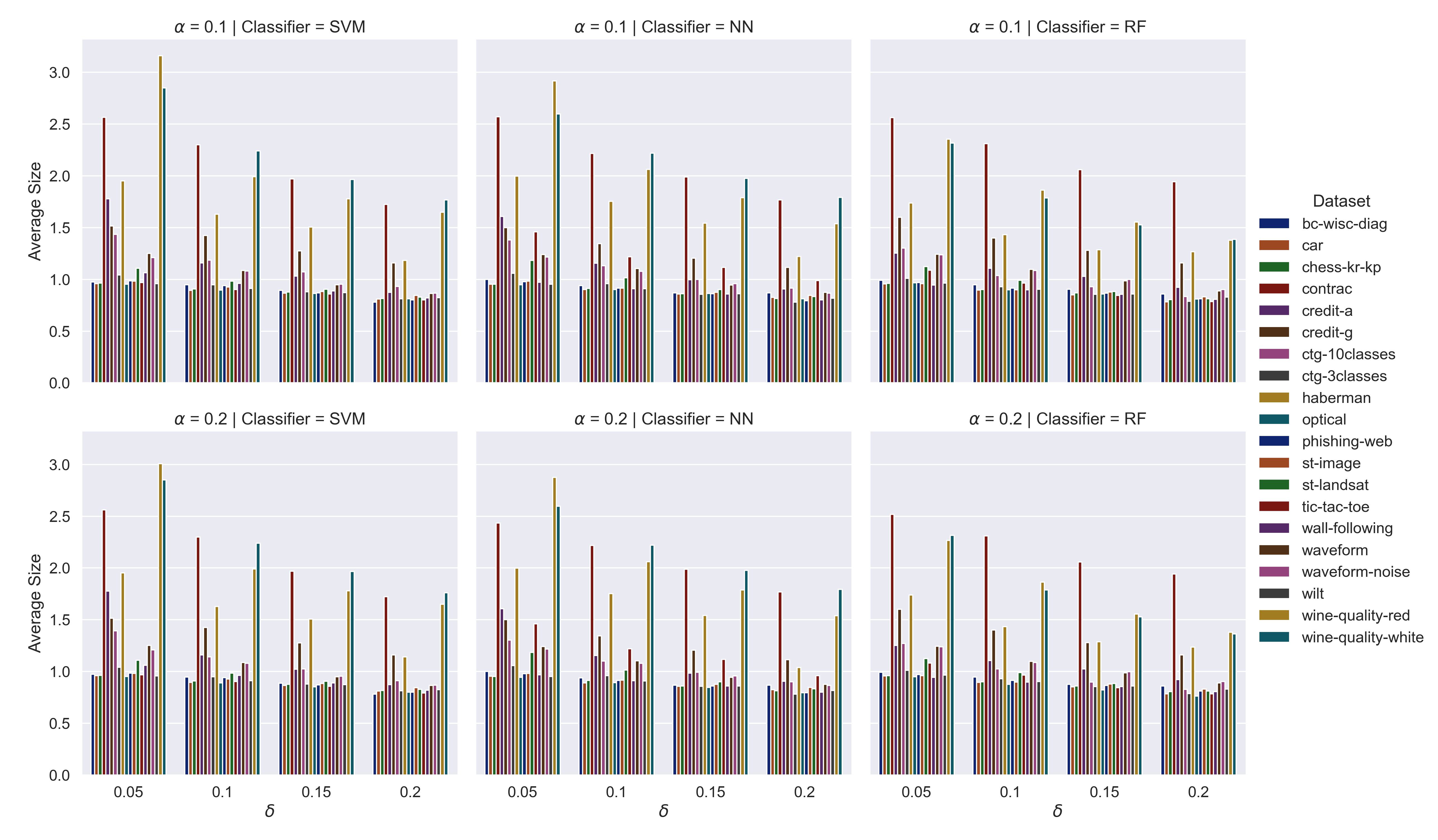}
\caption{Bar plots of the average sizes of prediction sets vs. $\delta = 0.05, 0.1, 0.15, 0.2$ on test data for classification with a class-varying loss. The first row corresponds to $\alpha = 0.1$ and the second row corresponds to $\alpha = 0.2$. Different columns represent different classifiers. The plots demonstrate the information in prediction sets. In general, large $\delta$ leads to small average size and different classifiers have different informational efficiency. }
\end{figure*}

\subsection{CLCP for high-impact weather forecasting}

\begin{figure*}[h]
\centering
\includegraphics[width = 0.75 \hsize]{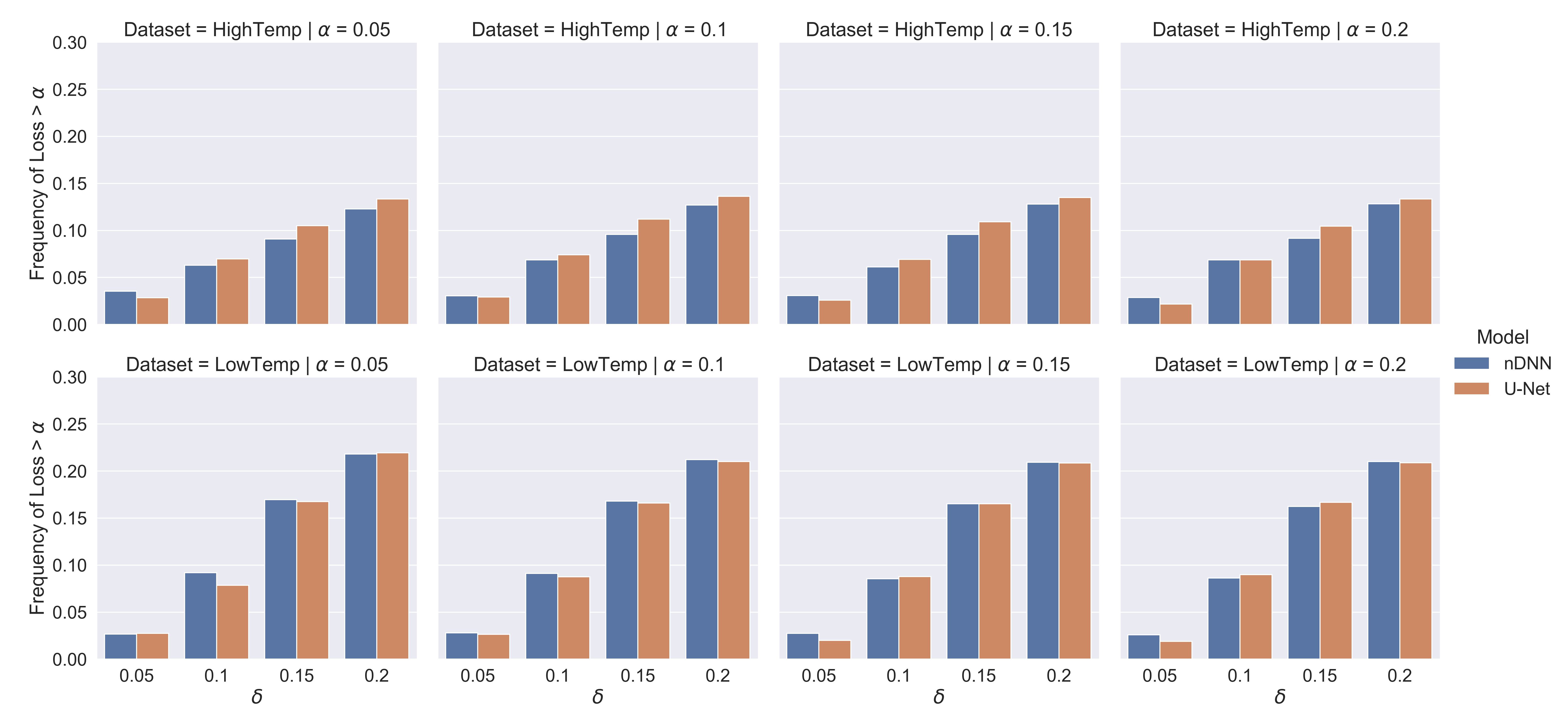}
\caption{Bar plots of the frequencies of the prediction losses being greater than $\alpha$ vs. $\delta = 0.05, 0.1, 0.15, 0.2$ on test data for high-impact weather forecasting. The first row corresponds to HighTemp and the second row corresponds to LowTemp. Different columns represent different $\alpha$. All bars are near or below the preset $\delta$, which confirms the controlling guarantee of CLCP empirically.}
\end{figure*}

\begin{figure*}[h]
\centering
\includegraphics[width = 0.75 \hsize]{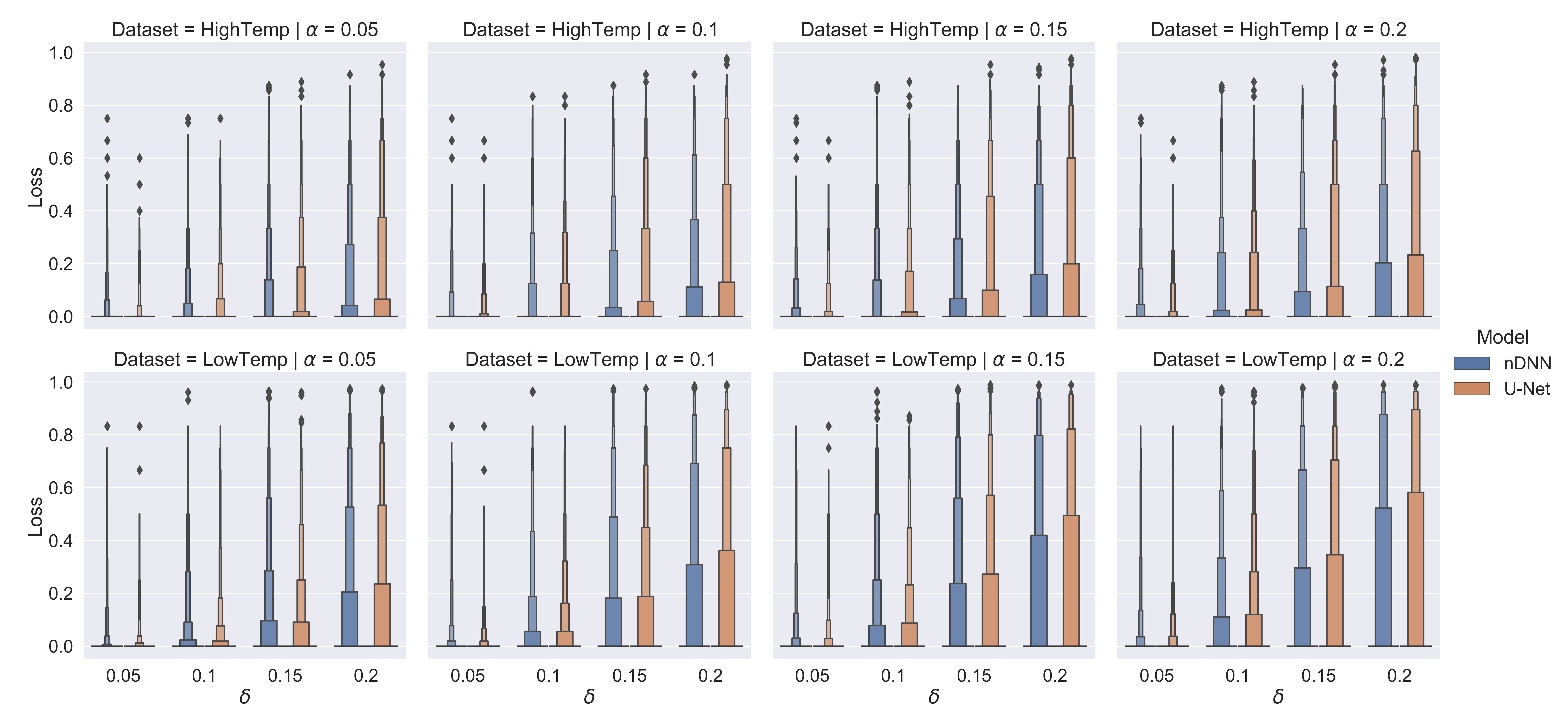}
\caption{Boxen plots of the prediction losses vs. $\delta = 0.05, 0.1, 0.15, 0.2$ on test data for high-impact weather forecasting. The first row corresponds to HighTemp and the second row corresponds to LowTemp. Different columns represent different $\alpha$. The loss distributions are controlled by $\alpha$ and $\delta$ properly to obtain the empirical validity in Fig. 3.}
\end{figure*}

\begin{figure*}[h]
\centering
\includegraphics[width = 0.75 \hsize]{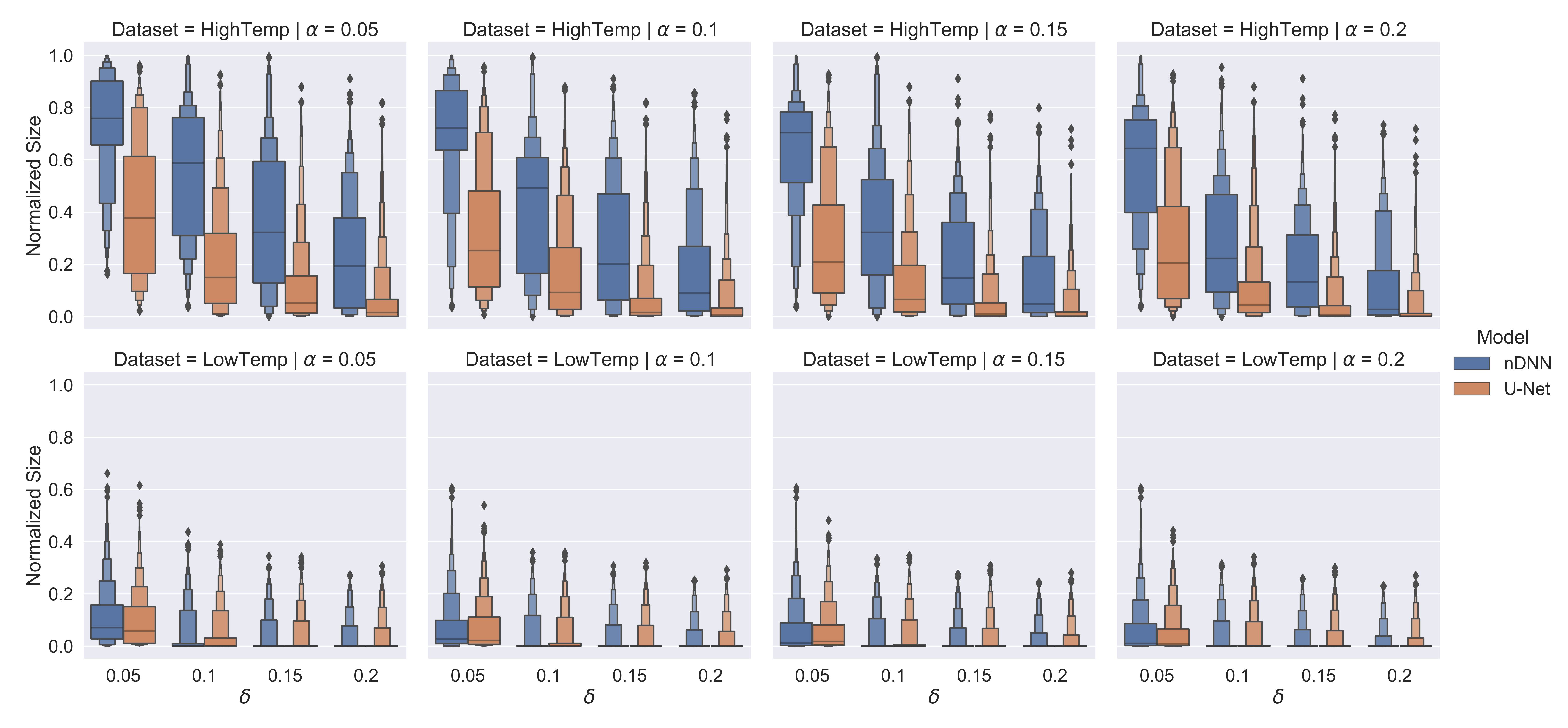}
\caption{Boxen plots for the distributions of normalized sizes of prediction sets vs. $\delta = 0.05, 0.1, 0.15, 0.2$ on test data for high-impact weather forecasting. The first row corresponds to HighTemp and the second row corresponds to LowTemp. Different columns represent different $\alpha$. U-Net performs better than nDNN, which indicates the importance of careful design of the underlying algorithm. } \end{figure*}

The remaining experiments apply CLCP to weather forecasting problems. Here we concentrate on postprocessing of the forecasts made by numerical weather prediction (NWP) models \cite{vannitsem2018statistical} \cite{vannitsem2021statistical}. NWP models use equations of atmospheric dynamics and estimations of current weather conditions to do weather forecasting, which is the mainstream weather forecasting technique nowadays especially for forecasting beyond $12$ hours. Many errors affect the performance of NWP models, such as the estimation errors of initial conditions and the approximation errors of NWP models, leading to the research topic about postprocessing the forecasts of NWP models. Most postprocessing methods are built on some learning process, which takes the forecasts of NWP models as inputs and the observations of weather elements or events as outputs.

In this paper, we use CLCP to postprocess the ensemble forecasts with the control forecast and $50$ perturbed forecasts issued by the NWP model from European Centre for Medium-Range Weather Forecasts (ECMWF) \cite{palmer2019ecmwf}, which are obtained from the THORPEX Interactive Grand Global Ensemble (TIGGE) dataset \cite{cisl_rda_ds330}. We focus on $2$-m maximum temperature and minimum temperature between the forecast lead times of $12$nd hour and $36$th hour with the forecasts initialized at $0000$ UTC. The forecast fields are grided with the resolution of $0.5^{\circ} \times 0.5^{\circ}$ and the corresponding label fields with the same resolution are extracted from the ERA5 reanalysis data \cite{hersbach2020era5}.
The area ranges from $109^{\circ}$E to $122^{\circ}$E in longitude and from $29^{\circ}$N to $42^{\circ}$N in latitude, covering the main parts of North China, East China and Central China, whose grid size is $27 \times 27$. The ECMWF forecast data and ERA5 reanalysis data are collected from $2007$ to $2020$ ($14$ years).

We first consider high-impact weather forecasting, which is to forecast whether a high-impact weather exists for each grid and can be seen as a point-wise classification problem or image segmentation problem for computer vision. The high-impact weather we consider is whether the $2$-m maximum temperature is above $\SI{35}{\degreeCelsius}$ or the $2$-m minimum temperature is below $\SI{-15}{\degreeCelsius}$ for each grid. These two cases are treated as high temperature weather or low temperature weather in China, which make meteorological observatories issue high temperature warning or low temperature warning respectively.

The prediction sets and the loss function used for high-impact weather forecasting are the same as those for image segmentation in \cite{angelopoulos2022conformal}.
Taking the ensemble forecast fields of the NWP model as input $x$, the corresponding label $y$ is a set of grids having high-impact weather, which can be seen as a segmentation problem for high-impact weather. Therefore, we first train a segmentation neural network $f(x)$, where $f_{(p,q)}(x)$ is the estimated probability of the grid $(p,q)$ having high-impact weather. Then the set-valued function $C_{\lambda} $ can be constructed as
\begin{equation}
C_{\lambda}(x) = \{(p,q): f_{(p,q)}(x) \geq 1 - \lambda\},
\end{equation}
and the loss function is
\begin{equation}
L(y, C) = 1 - \frac{|y \cap C|}{| y|},
\end{equation}
which measures the ratio of the prediction sets failing to do the warning. We use CLCP with the prediction set and the loss function above to do high temperature and low temperature forecasting respectively.

\subsubsection{Dataset for high temperature forecasting}

The reanalysis fields of $2$-m maximum temperature were collected from ERA5 and the label fields were calculated based on whether the $2$-m maximum temperature is above $\SI{35}{\degreeCelsius}$. To make the loss function take finite values, we only collected the data whose label fields have at least one high temperature grid to do this empirical study, which resulted in $1200$ samples in total, i.e., $1200$ ensemble forecasts from the NWP model of ECMWF and corresponding label fields calculated from ERA5. We name this dataset as HighTemp.

\subsubsection{Dataset for low temperature forecasting}

The dataset for testing CLCP for low temperature weather forecasting was constructed in a similar way.  The reanalysis fields of $2$-m minimum temperature were collected from ERA5 and the label fields were calculated based on whether the $2$-m minimum temperature is below $\SI{-15}{\degreeCelsius}$. We only collected the data whose label fields have at least one low temperature grid to do this empirical study, which resulted in $1233$ samples in total. We name this dataset as LowTemp.

For each dataset, the same process was used to conduct the experiment as Section IV-A ,  i.e., all forecasts from the NWP model were normalized to $[0,1]$ by min–max normalization, and we used $20\%$ of the data for testing and $80\%$ and $20\%$ of the remaining data for training and calibration respectively. We employed two fully convolutional neural networks \cite{li2021survey} for binary image segmentation as our underlying algorithms. One was U-Net \cite{ronneberger2015u} with the same structure as that in \cite{gronquist2021deep}, whose numbers of hidden feature maps were all set to $32$. The other was the naive deep neural network (nDNN) with the same encoder-decoder structure as the U-Net without skip-connections, i.e., the U-Net removing skip-connections. We use these two neural networks to show that the design of the underlying algorithm is necessary for better performance, as U-Net fuses multi-scale features and nDNN does not. The data for training U-Net and nDNN were further partitioned into the validation part ($10\%$) for model selection and proper training part ($90\%$) for updating the parameters. Adam optimization \cite{kingma2014adam} was used for training. The learning rate was set to $0.0001$ and the number of epochs was set to $50$. After training $50$ epochs, the model with lowest binary cross entropy on validation data was used for formula (10) to construct prediction sets, where $\lambda$ is searched from $1$ to $0$ with step size $0.01$. The experiments of using CLCP for the loss function as formula (11) were conducted $10$ times and the prediction results on test set are shown in Fig. 3, Fig. 4 and Fig. 5.

Fig. 3 also shows the bar plots of the frequencies of the prediction losses being greater than $\alpha$ for $\delta = 0.05, 0.1, 0.15$ and $0.2$. Four columns stand for the cases where $\alpha = 0.05, 0.1, 0.15$ and $0.2$ respectively. It can be seen that for the two datasets HighTemp and LowTemp, all bars are near or below the preset $\delta$, which verifies formula (7) empirically. Fig. 4 further shows the distributions of the losses for different $\delta$ and different $\alpha$ using boxen plots, which contain more information than box plots by drawing narrow boxes for tails. It can be seen that larger $\alpha$ and $\delta$ lead to larger losses, which is reasonable since large $\alpha$ and $\delta$ relax the constraint on prediction losses. We measure the informational efficiency of the prediction set $C_{\lambda^*}(x)$ using its normalized size defined as $|C_{\lambda^*}(x)| /PQ$, where $P$ and $Q$ are the numbers of the vertical and the horizontal grids respectively. The distributions of normalized sizes in Fig. 5 show that U-Net is more informationally efficient than nDNN, which indicates that design of the underlying algorithm is important for CLCP. Different $\alpha$ and $\delta$ lead to different normalized sizes, implying the trade-off among the preset loss level $\alpha$, confidence level $1-\delta$ and informational efficiency of the prediction sets. By choosing $\alpha$ and $\delta$ properly, the prediction sets of CLCP can have reasonable sizes. Also, we can see that forecasting low temperature is somehow easier than high temperature with the fact that for the same $\alpha$ and $\delta$, the normalized sizes of forecasting low temperature are distributed lower than the ones of forecasting high temperature, indicating the need of design of the underlying algorithms to improve performance for forecasting high temperature.

\subsection{CLCP for maximum temperature and minimum temperature forecasting}

\begin{figure*}[h]
\centering
\includegraphics[width = 0.75 \hsize]{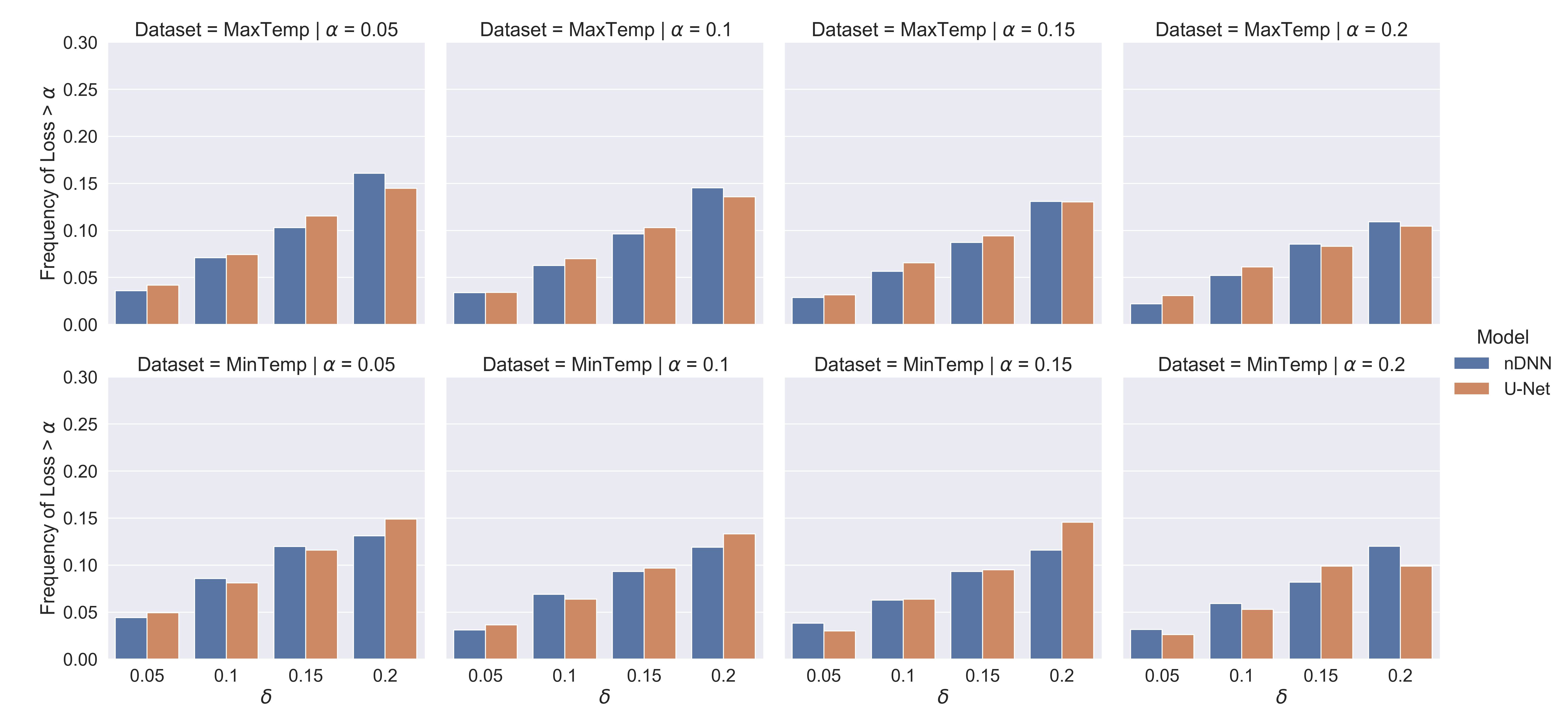}
\caption{Bar plots of the frequencies of the prediction losses being greater than $\alpha$ vs. $\delta = 0.05, 0.1, 0.15, 0.2$ on test data for maximum temperature and minimum temperature forecasting. The first row corresponds to MaxTemp and the second row corresponds to MinTemp. Different columns represent different $\alpha$. All bars are near or below the preset $\delta$, which confirms the controlling guarantee of CLCP empirically.}
\end{figure*}

\begin{figure*}[h]
\centering
\includegraphics[width = 0.75 \hsize]{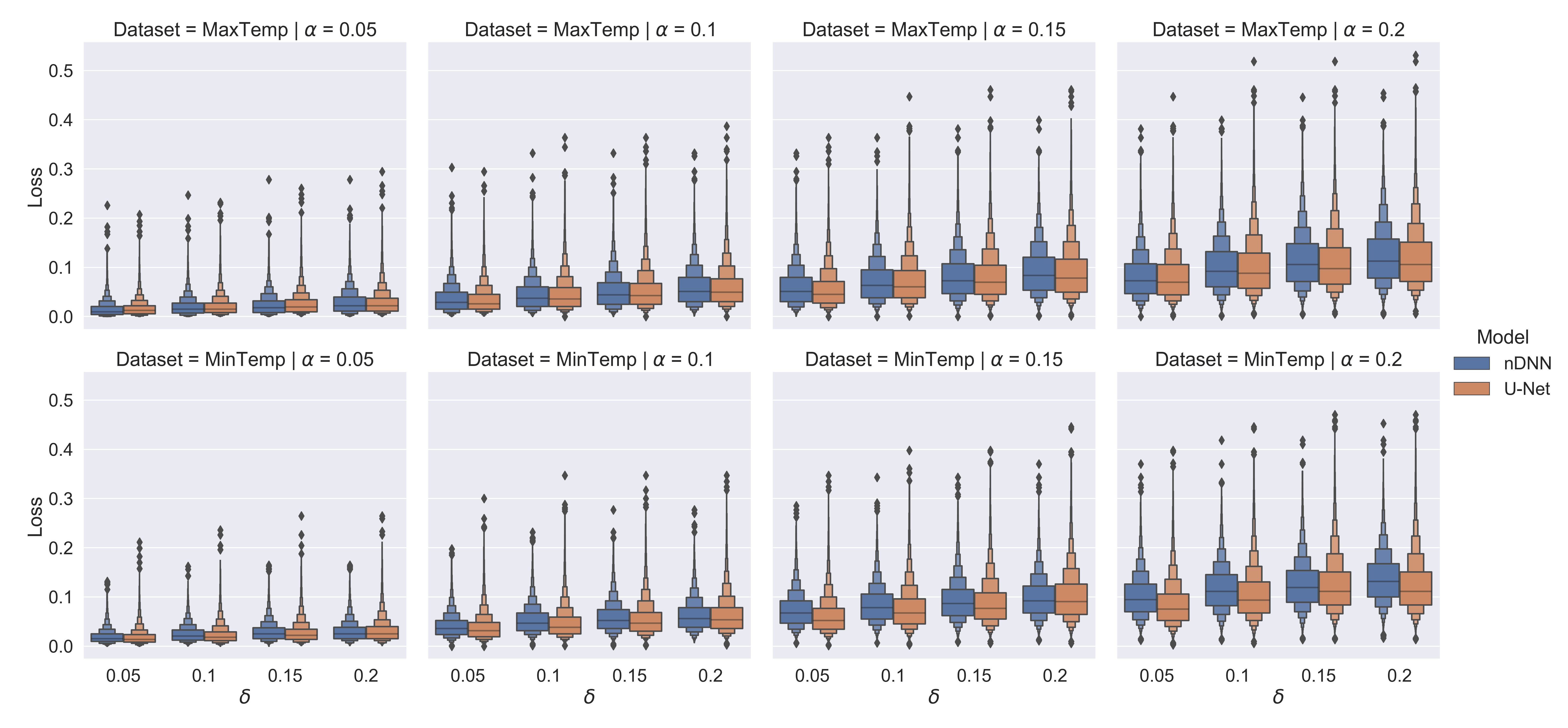}
\caption{Boxen plots of the prediction losses vs. $\delta = 0.05, 0.1, 0.15, 0.2$ on test data for maximum temperature and minimum temperature forecasting. The first row corresponds to MaxTemp and the second row corresponds to MinTemp. Different columns represent different $\alpha$. The loss distributions are controlled by $\alpha$ and $\delta$ properly to obtain the empirical validity in Fig. 6.}
\end{figure*}

\begin{figure*}[h]
\centering
\includegraphics[width = 0.75 \hsize]{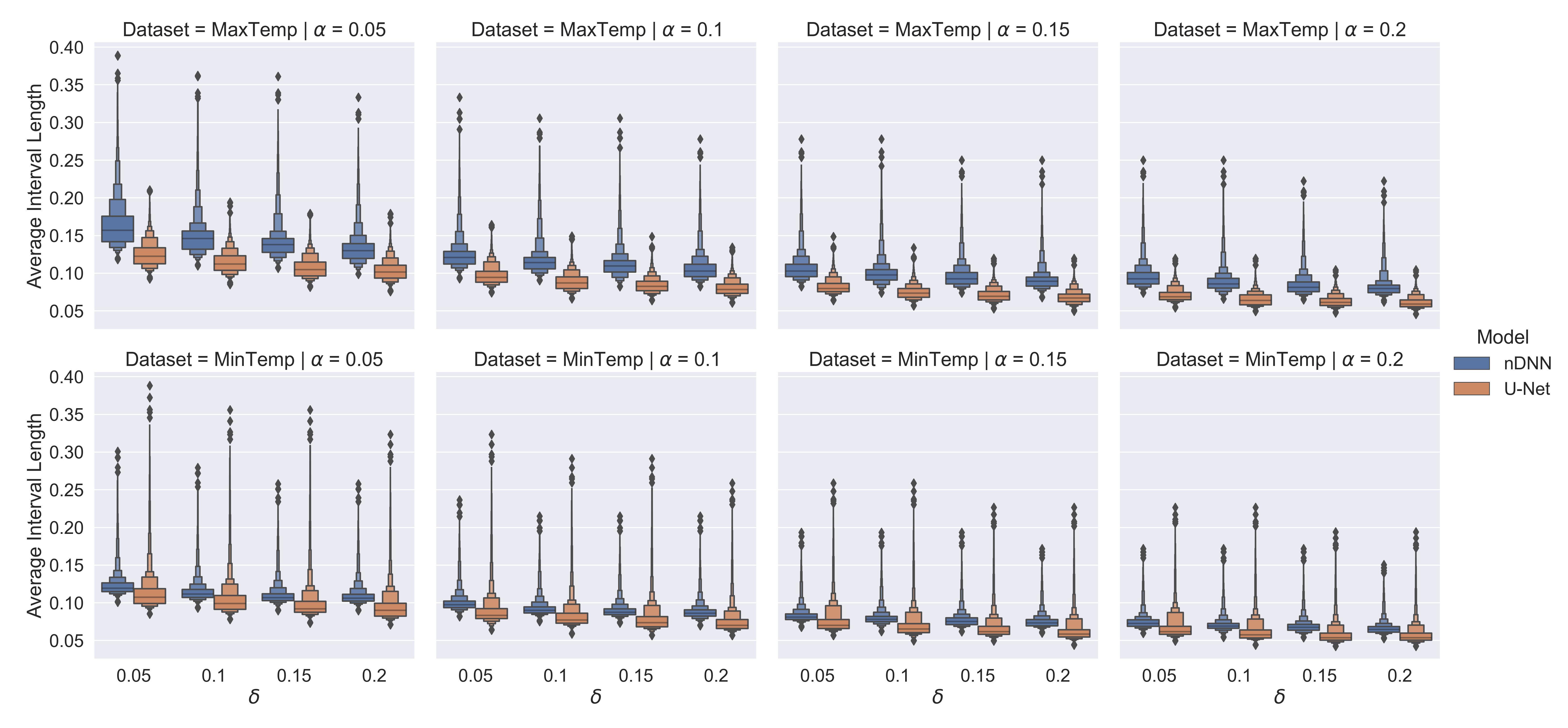}
\caption{Boxen plots for the distributions of average interval length vs. $\delta = 0.05, 0.1, 0.15, 0.2$ on test data for maximum temperature and minimum temperature forecasting. The first row corresponds to MaxTemp and the second row corresponds to MinTemp. Different columns represent different $\alpha$. U-Net performs better than nDNN, which indicates the importance of careful design of the underlying algorithm. }
\end{figure*}

This section focuses on using CLCP to forecast the $2$-m maximum temperature or minimum temperature value for each grid, which is a point-wise regression problem or image-to-image regression problem. To construct the prediction sets, we follow the procedure proposed in \cite{angelopoulos2022image} and train the neural network with $3$ output channels jointly predicting the point-wise $0.05$, $0.5$ and $0.95$ quantiles of the fields using quantile regression \cite{koenker1978regression} \cite{angelopoulos2022image}, which are denoted by $f^{0.05}(x)$, $f^{0.5}(x)$ and $f^{0.95}(x)$. Then the prediction set $C_{\lambda}(x)$  is equal to
\begin{equation}\nonumber
\Big \{y: y_{(p,q)} \in  [f^{0.5}_{(p,q)}(x) - \lambda \Delta^{-}_{(p,q)}(x), f^{0.5}_{(p,q)}(x) + \lambda \Delta^{+}_{(p,q)}(x) ] \Big \},
\end{equation}
where
\begin{equation}\nonumber
\Delta^{-}(x) = \max \{f^{0.5}(x) - f^{0.05}(x), 10^{-6}\},
\end{equation}
\begin{equation}\nonumber
\Delta^{+}(x) = \max \{f^{0.95}(x) - f^{0.5}(x), 10^{-6}\},
\end{equation}
and $\max$ is a point-wise operator making $\Delta^{-}$ and $\Delta^{+}$ at least $10^{-6}$. This prediction set is a prediction band for the output field, whose prediction interval at grid $(p,q)$ is
\begin{equation}\nonumber
[f^{0.5}_{(p,q)}(x) - \lambda \Delta^{-}_{(p,q)}(x), f^{0.5}_{(p,q)}(x) + \lambda \Delta^{+}_{(p,q)}(x) ]
\end{equation}
with the point-wise width being an increasing function of $\lambda$. This construction was proposed in \cite{angelopoulos2022image} for image-to-image regression and we use the same loss function in \cite{angelopoulos2022image} measuring miscoverage rate of a prediction band $C$ for a field $y$, which can be formalized as
\begin{equation}\nonumber
L(y, C) = \frac{1}{PQ}\Big | \Big \{(p,q): y_{(p,q)} \notin  C_{(p,q)} \Big \} \Big |,
\end{equation}
where $C_{(p,q)} $ is the prediction interval at grid $(p,q)$ for prediction band $C$.

All of the data collected from $2007$ to $2020$ were used, leading to $4945$ samples for each forecasting application and the datasets are named as MaxTemp and MinTemp respectively.
The experimental design is the same as that in Section IV-B, except that we also normalized the label for each grid to $[0,1]$ by min–max normalization, used quantile loss for model selection  and we searched for $\lambda^*$ with two steps. First we found two values $\lambda_1$ and $\lambda_2$ from $\{100, 10, 1, 0.1, 0.01, ...\}$ such that $Q_{1-\delta}^{(n)}(\lambda_1) \leq \alpha$ and $Q_{1-\delta}^{(n)}(\lambda_2) > \alpha$. Then we searched for $\lambda^*$ from $100$ values starting with $\lambda_1$ and ending with $\lambda_2$ using a common step size. The experimental results are recorded in Fig. 6, Fig 7 and Fig 8.

Although the set predictors and the loss function used in this section are different from those in Section IV-B, the experimental results and conclusions are similar. From Fig. 6, we can see that the frequencies of the prediction losses being greater than $\alpha$ are controlled by $\delta$, which also verifies formula (7) empirically. Larger $\alpha$ and $\delta$ lead to larger losses, which is shown in Fig. 7.
Here we use the following average interval length
\begin{equation}\nonumber
\frac{1}{PQ}\sum_{p = 1}^{P}\sum_{q = 1}^{Q}\lambda^* (\Delta^{+}_{(p,q)}(x)- \Delta^{-}_{(p,q)}(x))
\end{equation}
to measure the informational efficiency of the prediction set  $C_{\lambda^*}(x)$  and Fig. 8 also depicts the trade-off among the preset loss level $\alpha$, confidence level $1-\delta$ and informational efficiency of the prediction sets and indicates that better design of underlying algorithms leads to better performance.

\section{Conclusion}

This paper extends conformal prediction to the situation where the value of a loss function needs to be controlled, which is inspired by risk-controlling prediction sets and conformal risk control approaches. The loss-controlling guarantee is proved in theory with the assumption of exchangeability and is empirically verified for different kinds of applications including classification with a class-varying loss and weather forecasting. Different from conformal prediction, conformal loss-controlling prediction approach proposed in this paper has two preset parameters $\alpha$ and $\delta$, which guarantees that the prediction loss is not greater than $\alpha$ with confidence $1 - \delta$. Both parameters impose restrictions on prediction sets and should be set based on specific applications. Despite loss-controlling guarantee, informational efficiency of the prediction sets built by conformal loss-controlling prediction is highly related to underlying algorithms, which has been shown in empirical studies. Since this is a rather new topic, the underlying algorithms and the way of constructing set predictors are inherited from conformal risk control. This leaves the important question on how to build informationally efficient set predictors in an optimal way, which is one of our further researches in the future.




\ifCLASSOPTIONcaptionsoff
  \newpage
\fi




\bibliography{ref}

\bibliographystyle{IEEEtran}

\end{document}